\documentclass{article}

\usepackage{arxiv}

\usepackage[utf8]{inputenc} 
\usepackage[T1]{fontenc}    
\usepackage{hyperref}       
\usepackage{url}            
\usepackage{booktabs}       
\usepackage{amsfonts}       
\usepackage{nicefrac}       
\usepackage{microtype}      
\usepackage{lipsum}

\usepackage{cite}
\usepackage{amsmath,amssymb,amsfonts}
\usepackage{graphicx}
\usepackage{textcomp}
\usepackage{xcolor}
\usepackage{xspace}
\usepackage{tikz}
\usepackage{pgfplots}
\usetikzlibrary{arrows,backgrounds,decorations,decorations.pathmorphing,positioning,fit,automata,shapes,snakes,patterns,plotmarks,calc}
\usepackage{subfig}
\usepackage[vlined,ruled,linesnumbered]{algorithm2e}
\usepackage{stmaryrd}
\usepackage[inline]{enumitem}
\usepackage{booktabs}
\usepackage{mathtools,amssymb,amsthm,enumitem}

\newtheorem{definition}{Definition}[section]
\newtheorem{example}{Example}[section]
\newtheorem{lemma}{Lemma}[section]
\newtheorem{remark}{Remark}
\newtheorem{theorem}{Theorem}

\newcommand{\f}{\varphi}
\newcommand{\STL}{STL\xspace}

\newcommand{\PSTL}{PSTL\xspace}
\newcommand{\Intvl}{I}
\newcommand{\sPara}{c}

\newcommand{\val}{\nu}
\newcommand{\vx}{\mathbf{x}}
\newcommand{\vy}{\mathbf{y}}
\newcommand{\x}{\mathbf{x}}
\newcommand{\G}{\mathbf{G}}
\newcommand{\F}{\mathbf{F}}
\newcommand{\alw}{\mathbf{G}}

\newcommand{\U}{\mathbf{U}}
\newcommand{\true}{\mathit{true}}

\newcommand{\aand}{\,\wedge\,}

\newcommand{\Reals}{\mathbb{R}}

\newcommand{\PosReals}{\mathbb{R}^{\ge 0}}
\newcommand{\domain}{\mathcal{D}}
\newcommand{\setof}[1]{\left\{#1\right\}}

\newcommand{\params}{\mathcal{P}}
\newcommand{\timedomain}{T}

\newcommand{\paramspace}{{\domain_\params}}

\newcommand{\validitydomain}[1]{\mathcal{V}(#1)}

\newcommand{\mypara}[1]{\vspace{0.3em} \noindent {\bf #1.\ }}

\newcommand{\assign}{\leftarrow}

\newcommand{\oreq}{\varphi_{\mathrm{out}}}
\newcommand{\ireq}{\varphi_{\mathrm{in}}}

\newcommand{\support}{\mathsf{supp}}

\newcommand{\good}{\mathrm{good}}
\newcommand{\bad}{\mathrm{bad}}
\newcommand{\cex}{\mathrm{cex}}
\newcommand{\candidate}{\mathrm{proposed}}

\newcommand{\traces}{\mathcal{T}}

\newcommand{\featureValues}{\mu}
\newcommand{\labels}{\ell}
\newcommand{\simulink}{Simulink\textsuperscript{\textregistered}}

\newcommand{\vu}{\mathbf{u}}
\newcommand{\y}{\mathbf{y}}
\newcommand{\param}{\mathbf{p}}

\title{Mining Environment Assumptions for Cyber-Physical System Models}

\author{
  Sara Mohammadinejad\\
  University of Southern California\\
  \texttt{saramoha@usc.edu} \\
   \And
  Jyotirmoy V. Deshmukh\\
  University of Southern California\\
  \texttt{jdeshmuk@usc.edu} \\
     \And
  Aniruddh G. Puranic\\
  University of Southern California\\
  \texttt{puranic@usc.edu} \\
}
\begin{document}
\maketitle

\begin{abstract}
Many complex cyber-physical systems can be modeled as heterogeneous
components interacting with each other in real-time. We assume that
the correctness of each component can be specified as a requirement
satisfied by the output signals produced by the component, and that
such an {\em output guarantee} is expressed in a real-time temporal
logic such as Signal Temporal Logic (STL).  In this paper, we
hypothesize that a large subset of {\em input signals} for which the
corresponding output signals satisfy the output requirement can also
be compactly described using an STL formula that we call the {\em
environment assumption}.  We propose an algorithm to mine such an
environment assumption using a supervised learning technique.
Essentially, our algorithm treats the environment assumption as a
classifier that labels input signals as {\em good} if the
corresponding output signal satisfies the output requirement, and as
{\em bad} otherwise. Our learning method simultaneously learns the
structure of the STL formula as well as the values of the numeric
constants appearing in the formula.\footnote{If the structure or template of the STL formula is given based on user-defined domain knowledge, learning the parameters of the template is trivial, and our method is able to do that. The seminal works in \cite{vazquez2017logical, vazquez2018time, jin2015mining} focus on learning the values of parameters for a user-defined template PSTL formula.} To achieve this, we combine a
procedure to systematically enumerate candidate Parametric STL (PSTL)
formulas, with a decision-tree based approach to learn parameter
values. We demonstrate experimental results on real world data from
several domains including transportation and health care.
\end{abstract}

\section{Introduction}
Autonomous cyber-physical systems such as self-driving cars, unmanned
aerial vehicles, general purpose robots, and medical devices can often
be modeled as a system consisting of heterogeneous components. Each of
these components could itself be quite complex: for example, a
component could contain design elements such as a model predictive
controller, a deep neural network, rule-based control,
high-dimensional lookup tables to identify operating regime, etc.
Understanding the high-level behavior of such components at an
abstract, behavioral level is thus a significant challenge.  The
complexity of individual components makes compositional reasoning
about global properties a difficult task.  Contract-based reasoning
\cite{nuzzo2014relational,li2017stochastic} is a potential approach
for compositional reasoning of such complex component-based CPS
models. Here, a design component $C$ is modeled in terms of {\em
environment assumptions}, i.e., assumptions on the timed input traces
to $C$, and {\em output guarantees}, i.e.  properties satisfied by the
corresponding model outputs.  A big challenge is that designers do not
often articulate such assumptions and guarantees using logical,
machine-checkable formalisms \cite{yamaguchi2016combining}. 

Recently, there is considerable momentum to express formal
requirements of design components using real-time temporal logics such
as Signal Temporal Logic (STL)
\cite{jin2014powertrain,RoehmEtAl2015arch,HoxhaAF14arch1,kapinski2016st,bartocci2013learning,cameron2015towards}.
Typical STL requirements express families of excitation patterns on
the model inputs or designer-specified pre-conditions that guarantee
desirable behavior of the model outputs \cite{ferrere2019interface}.
In this paper, we consider the dual problem: {\em Given an output
requirement $\oreq$, what are the assumptions on the model
environment, i.e., input traces to the model, that guarantee that the
corresponding output traces satisfy $\oreq$?} Drawing on the
terminology from \cite{jin2015mining,hoxha2018mining}, we call this
problem the {\em assumption mining} problem. 

We propose an approach that reduces the assumption mining problem to
{\em supervised learning}. We assume that input traces can be assigned
labels {\em desirable} and {\em undesirable} based on whether the
corresponding output traces satisfy or violate $\oreq$ respectively.
A potential approach is to then use off-the-shelf supervised learning
methods for time-series data from the machine learning (ML) community.
However, such techniques typically train discriminators in high
dimensional spaces which may not be human-interpretable
\cite{jones2014anomaly}.  Interpretability is an important factor for
safety-critical applications as components are usually developed by
independent design teams, and articulating the assumptions and
guarantees in an interpretable format can reduce downstream bugs
introduced during system integration. 

In this paper, we assume that environment assumptions can be expressed
in STL. The use of STL to express such assumptions has been explored
before in \cite{ferrere2019interface,jin2014powertrain}. However,
there is no existing work on automatically inferring such assumptions
from component models. The primary contribution of this paper is a new
algorithm to mine environment assumptions (expressed in STL). Our
counterexample-guided inductive synthesis algorithm systematically
enumerates parametric STL (\PSTL) formulas, and attempts to find
parameter valuations such that the resulting formula $\ireq$
classifies the given labeled input traces with high accuracy. This
step of our algorithm uses a decision tree based algorithm for
learning the parameter valuations for a \PSTL formula that lead to
good classification accuracy. Our choice of the feature space for the
decision tree classifier allows us to extract an STL formula from the
decision tree itself. In the next step, we make use of a falsification
procedure to check if there exists an input trace to the model that
satisfies $\ireq$ but the corresponding output does not satisfy
$\oreq$. If such a trace exists, we resume the enumerative search for
an accurate STL-based classifier. 

To summarize, our key contributions are as follows:
\begin{itemize}
\item We propose a new algorithm to mine environment assumptions (expressed in STL) automatically. 
\item As our algorithm systematically increases the syntactic complexity of the \PSTL formulas, it uses the Occam's Razor principle to learn environment assumptions, i.e., it attempts to learn STL classifiers that are short, and hence simple and more interpretable\footnote{We prevent excessive generalization and simplification by assuming a threshold on the accuracy of the learned STL formula.}. 
\item We demonstrate the capability of our assumption mining algorithm on a few benchmark models.
\end{itemize}

\section{Preliminaries}
\begin{definition}[Timed Traces]
A timed trace defines a function from a time domain $\timedomain$
(which is a finite or infinite collection of ordered time instants) to
a non-empty set $\domain$ equipped with a distance metric. 
\end{definition}

In this paper, we restrict our attention to discrete timed traces,
where $\timedomain$ is essentially a finite subset of $\PosReals$ that
includes $0$, and $\domain$ is assumed to be some subset of
$\Reals^n$. A {\em trace variable} or a {\em signal} $\x$ is a
variable that evaluates to timed traces. We abuse notation and use
$\x(t)$ to denote the valuation of the trace variable $\x$ at time
$t$. The time domain associated with the trace variable $\x$ is
denoted by $\timedomain(\x)$.  We remark that the bold-face upright
$\x$ denotes a multi-dimensional signal, i.e. $\x = (x_1,\ldots,x_k)$,
where each $x_i$ is single-dimensional (i.e. their domain is a subset
of $\Reals$).  The dimension of $\x$ is $k$.  Next, we define the
notion of a dynamical model of a CPS component.

\begin{definition}[Dynamical Models of a CPS component]
A dynamical model $M_C$ of a CPS component $C$ is defined as a machine
containing a set of input signals (i.e. input trace variables) $\vu$,
output signals $\vy$, and state signals $\vx$. We assume that the
domains of $\vu$, $\vy$ and $\vx$ are $U$, $Y$ and $X$ respectively.
Let $\vx(0)$ denote the initial valuation of the state variables. The
dynamical model $M_C$ takes an input trace $\vu(t)$, an initial state
valuation $\vx(0)$ and produces an output trace $\vy(t)$, denoted as
$\vy(t) = M_C(\vu(t), \vx(0))$.
\end{definition}

We note that typically, there may be a state trace $\vx(t)$ denotes a
system trajectory that evolves according to certain dynamical
equations that depend on $\vx(\tau)$ for $\tau < t$ and $\vu(t)$.
Further, $\vy(t)$ is usually a function of $\vx(t)$ and $\vu(t)$.
However, for the purpose of this paper, we are only concerned with the
input/output behavior of $C$, and do not explicitly reason over
$\vx(t)$. We also assume that the initial valuation for the state
variables is fixed\footnote{This is not limiting as we can simply have
an input variable that is used to set an initial valuation for
$\vx(t)$ at time $0$ and is ignored for all future time points.}.
Further, if the component $C$ under test is obvious from the context,
we drop the subscript. Thus, we can simply state that $\vy(t) =
M(\vu(t))$ to denote the simplified view that the model $M$ is a
function over traces that maps input traces to output traces.

\mypara{Signal Temporal Logic (STL)} Signal Temporal Logic
\cite{maler2004monitoring} is a popular formalism for expressing
properties of real-valued signals. The simplest STL formulas are
atomic predicates over signals, that can be formulated as $f(\x) \sim
\sPara$, where $f$ is a function from $\domain$ to $\Reals$, $\x$ is a
signal, $\sim \in \setof{\geq, \leq,=}$,  and $\sPara \in \Reals$.
Logical and temporal operators are used to recursively build STL
formulas from atomic predicates and subformulas.  Logical operators
are Boolean operations such as $\neg$ (negation), $\wedge$
(conjunction), $\vee$ (disjunction), and $\implies$ (implication).
Temporal operators $\G$ (always), $\F$ (eventually) and $\U$ (until)
help express temporal properties over traces. Each temporal operator
is indexed by an interval $\Intvl := (a,b) \mid (a,b] \mid [a,b) \mid
[a,b]$, where $a, b \in \timedomain$.  Let $\sPara \in \Reals$, and
$\x$ be a signal, then \eqref{eq:stl_syntax} gives the syntax of STL.
\begin{equation}
\label{eq:stl_syntax}
\begin{array}{l}
\f := \true 
      \mid f(\x) \sim \sPara 
      \mid \neg\f 
      \mid \f_{1} \wedge \f_{2}\, 
      \mid \G \f\, 
      \mid \F\f\, 
      \mid \f_{1}\, \U_{\Intvl}\, \f_{2}~
\end{array}
\end{equation}

\begin{definition}[Support Variables of a Formula]
Given an STL formula $\varphi$, the support variables of $\varphi$ is the set of signals appearing in atomic predicates in any
subformula. We denote support of $\varphi$ by $\support(\varphi)$.
\end{definition}

The semantics of \STL can be defined in terms of the Boolean
satisfaction of a formula by a timed trace, or in terms of a function
that maps an STL formula and a timed trace to a numeric value known as
the {\em robustness value}.  If a trace $\x(t)$ satisfies a formula
$\varphi$, then we denote this relation as $\x(t) \models \varphi$.
We briefly review the {\em quantitative semantics} of STL from
\cite{donze2010robustness}, as we use it extensively in this
paper.

Formally, the robustness value approximates the signed distance of a
trace from the set of traces that marginally satisfy or violate the given
formula.  Technically, in \cite{donze2010robustness} the authors
define a robustness signal $\rho$ that maps an STL formula $\varphi$
and a trace $\x$ to a number at each time $t$ that denotes an
approximation of the signed distance of the suffix of $\x$ starting at
time $t$ w.r.t. traces satisfying or violating $\varphi$.  The
convention is to call the value at time $0$ of the robustness signal
of the top-level STL formula as the robustness value.  This definition has the property that if a trace
has positive robustness value then it satisfies the top-level formula,
and violates the formula if it has a negative robustness value. 
\begin{equation*}
\label{eq:stl_q_semantics}
\begin{array}{rcl}
\rho(f(\x) \ge \sPara,\x,t) & = & f(\x(t))-\sPara \\
\rho(\neg\varphi,\x,t) & = & -\rho(\varphi,\x,t)  \\
\rho(\varphi_1\aand\varphi_2,\x,t) & = & \min( \rho(\varphi_1,\x,t), \rho(\varphi_2,\x,t))  \\
\rho(\G_{\Intvl} \varphi,\x,t) & = & \inf_{t^\prime \in t\oplus \Intvl} \left(\rho(\varphi,\x,t^\prime)\right) \\
\rho(\F_{{\Intvl}} \varphi,\x,t) & = & \sup_{t^\prime \in t\oplus \Intvl} \left(\rho(\varphi,\x,t^\prime)\right) \\
\rho(\varphi_1\U_{\Intvl}\varphi_2,\x,t) & = & \sup\limits_{t^{\prime}\in t \oplus \Intvl} 
    \min \begin{pmatrix}\rho(\varphi_2,\x,t^{\prime}),\\ 
                        \inf_{t^{\prime\prime}\in [t,t^\prime) } \rho(\varphi_1,\x,t^{\prime\prime})
         \end{pmatrix}.

\end{array}
\end{equation*}

\noindent In the above, $\oplus$ denotes the Minkowski sum, i.e., $t
\oplus [a,b] = [t+a,t+b]$. Note that we only include the atomic
predicate of the form $f(\vx) \ge \sPara$, as any other atomic signal
predicate can be expressed using predicates of this form, negations
and conjunctions.

\begin{example}
Consider the signal $\x$, and the STL formulas $\f_1 = \G_{[0,10)} (x
\leq 3)$ and $\f_2 = \F_{[0, 10]} (x < -3)$. Consider a timed trace of
$\x$, where $\x(t) = \sin(2\pi t)$ (for some discrete set of time
instants $t \in [0,50]$).  This trace satisfies $\f_1$ because $\sin
(2 \pi t)$ never exceeds $3$ and violates $\f_2$ since $\sin (2 \pi t)
\ge -3$  for all $t$.  The robustness value of $\f_1$ with respect to
$\x(t)$ is the minimum of $3-\x(t)$ over $[0; 10)$, or $2$.  The
robustness value of $\f_2$ with respect to $\x(t)$ is  the maximum of
$-3-\x(t)$ over $[0,10]$ or $-2$ (see Fig.~\ref{fig:ro_discrete}).
\end{example}

\begin{figure}[!t]
\centering
\includegraphics[scale=0.15]{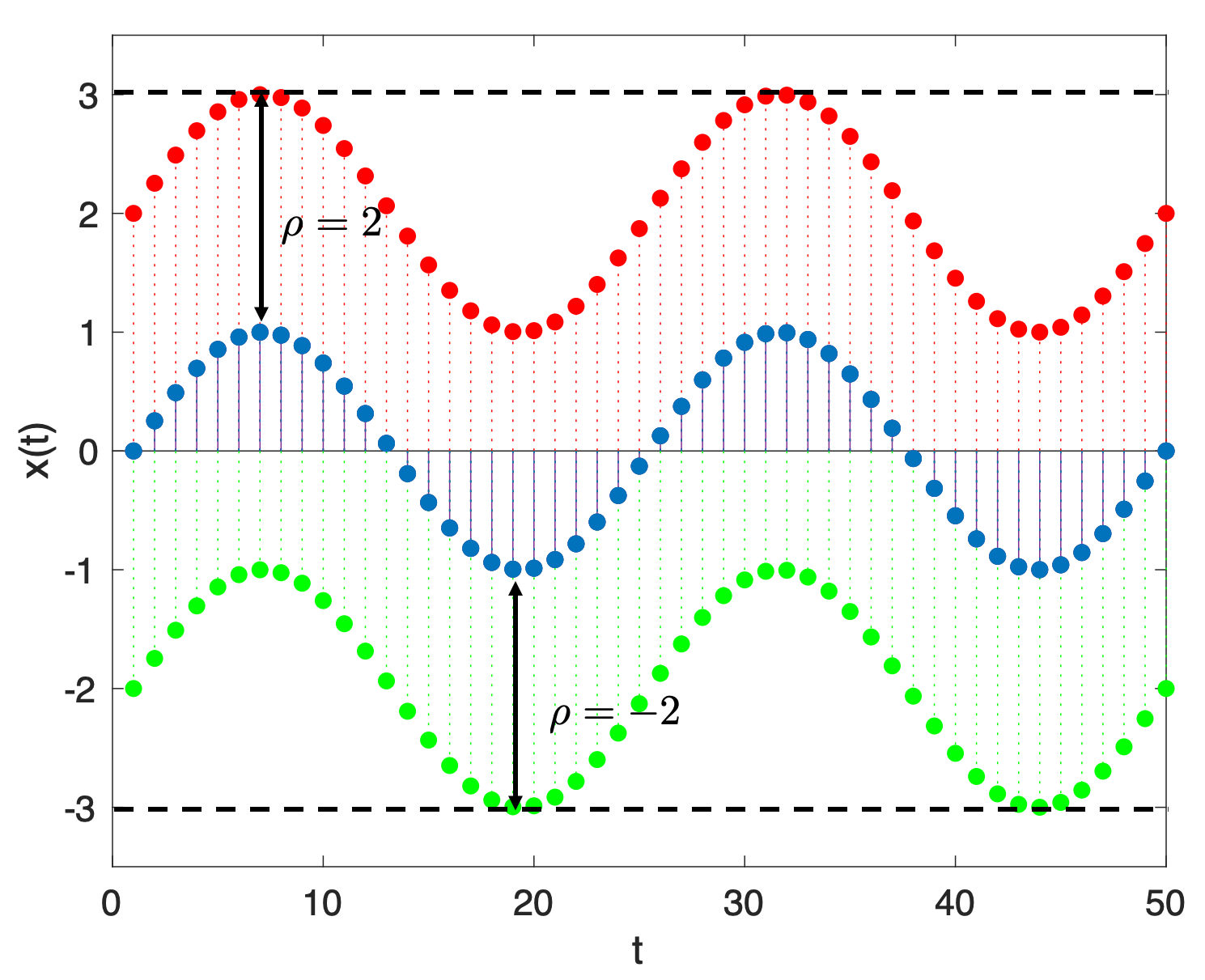}
\caption{The approximate signed distance between $\x(t)$ (blue
sinusoidal trace) and traces marginally violating $\f_1$ (i.e. the red trace) is $2$ (i.e. $\rho( \f_1,x,0) = 2$), and the approximate signed
distance between $\x(t)$ and traces marginally satisfying $\f_2$ (i.e. the green trace) is $-2$ (i.e. $\rho( \f_2,x,0) = -2$).
\label{fig:ro_discrete}}
\end{figure}

\mypara{Parametric Signal Temporal Logic (PSTL)} Parametric \STL
(\PSTL) \cite{asarin2011parametric} is an extension of \STL  where
constants appearing in atomic predicates or time intervals are
replaced by parameters.  \PSTL assumes a finite supply of parameter
variables $\params$, which come from two distinct sets: value-domain
parameter variables $\params_V$ and time-domain parameter variables
$\params_T$. We assume that the parameters in $\params_V$ can take
values in the set $V$, and those in $\params_T$ can take values in
$T$. Then, the parameter space of the \PSTL formula is $V \times
T$.
For any \PSTL formula, we associate a valuation function $\val$ that
maps parameter variables to some value in the parameter space.
Essentially, the valuation function has the effect of mapping a \PSTL
formula to an \STL formula with a specific valuation for the
parameters.

\begin{example} The property ``Always for the first $\tau$ seconds,
the trace $\x(t)$ is greater than some value $c_1$ and the signal
$\y(t)$ is less than $c_2$'' is written in \PSTL as: 
\[
\f(c_1,c_2,\tau) = \G_{[0,\tau]} (x > c_1 \aand y< c_2).
\]
In the above, $c_1$ and $c_2$ are value-domain parameter variables,
and $\tau$ is a time parameter variable.  Let $c_1 \in [1,2]$, $c_2
\in [0,3)$ and $\tau \in [0,10]$, the parameter space of $\f$ is
$[1,2] \times [0,3) \times[0,10]$.  The \STL formula $\G_{[0,6]} (x >
1.7 \aand y< 2)$ is obtained with the valuation $\val = \setof{\tau
\mapsto 6, c_1 \mapsto 1.7, c_2 \mapsto 2}$ applied to $\f$.
\end{example}

\subsection{Requirements and Assumptions}
In this section, we formalize the notion of output requirements and
input or environment assumptions.  

\begin{definition}[Output requirement]
Output requirement or $\oreq$ is an \STL formula that is satisfied by output traces of the system if their behavior is desirable and is not satisfied otherwise.
\end{definition}

\begin{definition}[Environment Assumption]
Given a dynamical component model $M_C$ = $(\vu,\vx,\vy)$, an output
requirement is an STL formula $\oreq$, where $\support(\oreq) = \vy$.
Given an output requirement $\oreq$, an STL formula $\ireq$ is called
an environment assumption if:
\begin{enumerate}
\item
$\support(\ireq) = \vu$, 
\item
$\forall \vu(t): \left(\vu(t) \models \ireq\right) \implies \left(M(\vu(t))
\models \oreq\right)$.
\end{enumerate}
\end{definition}

Essentially, an environment assumption is an STL property on the
input traces to the model that guarantees that the corresponding
output traces satisfy the output requirement $\oreq$.

\begin{example}
Consider a simple model $M$ that simply delays a given input signal
by $1$ second, i.e. the value of the output at time
$1$ is the value of the input signal at time $0$ (and the values of
the output in times $[0,1)$ are defined as some default output trace
value). Suppose the output requirement is $\G_{[1,100]} (\vy > 0)$, then
the property $\G_{[0,99]} (\vu > 0)$ is a valid environment assumption
for the model.
\end{example}

In software verification parlance, the environment assumption could be
viewed as a pre-condition over the input trace to the model that
guarantees an assertion on the output trace.

\section{Environment Assumption Mining}
In this section, we describe our overall approach to mine environment assumptions, and identify sub-problems that we will
address in subsequent sections. The central idea in our approach is a
counterexample-guided inductive synthesis (CEGIS) algorithm to mine
environment assumptions. The key steps of this process are shown in
Algorithm~\ref{algo:cegis_ea}.
\begin{algorithm}
\DontPrintSemicolon 
\KwIn{Input signal domain $U$, Output requirement $\oreq$, Input
signal time domain $\timedomain(\vu)$, Model $M = (\vu,\vy)$,
Simulation Budget $N$ for Falsification, Formula length limit
$\ell_\mathrm{max}$, Classification Accuracy $1-\epsilon$}
\KwOut{Environment Assumption $\ireq$}
$\traces$ = Sample input traces from $U$ using time instants from
$\timedomain(\vu)$ \nllabel{algoline:sample} \;
\ForEach{$\vu(t) \in \traces$}{
     \lIf{$M(\vu(t)) \models \oreq$}{ 
         $\traces_{\good}$ = $\traces_{\good} \cup \{\vu(t)\}$
         \nllabel{algoline:good}
     } \lElse {
         $\traces_{\bad}$ = $\traces_{\bad} \cup \{\vu(t)\}$ 
         \nllabel{algoline:bad}
     }
}
$\psi^\candidate$ = $\mathsf{EnumerateNextPSTL}()$ \nllabel{algoline:getNew1}\;
\While{$|\psi^\candidate| < \ell_\mathrm{max}$}{
     (accuracy, $\Upsilon[\psi^\candidate]$) =
     \ \ $\mathsf{DecisionTreeBasedSTLClassifier}(\psi^\candidate, \traces_{\good}, \traces_{\bad})$ 
     \nllabel{algoline:getdt} \;
     $\ireq^\candidate = \mathsf{GetSTL}(\Upsilon[\psi^\candidate])$ 
     \nllabel{algoline:getIreq} \;
     \If{accuracy $> 1-\epsilon$}{ \nllabel{algoline:accuracy}
         $\cex(t)$ = $\mathsf{Falsify}(\vy \models \oreq, N)$ \\
                \quad subject to $\vu(t) \models \ireq^\candidate$ \\
                    \quad \qquad \qquad\ $\vy(t) = M(\vu(t))$
                    \nllabel{algoline:getcex} \;
         \lIf{$\cex(t) \neq \emptyset$}{
             $\traces_{\bad}$ = $\traces_{\bad} \cup \{\cex(t)\}$
             \nllabel{algoline:add}
         } \lElse {
             \Return $\ireq^\candidate$ \nllabel{algoline:return}
         }
    } \Else {
         $\psi^\candidate$ = $\mathsf{EnumerateNextPSTL}()$
         \nllabel{algoline:getNew2}
     }
}
\caption{Environment Assumption Mining Algorithm\label{algo:cegis_ea}}
\end{algorithm}

We assume that the user provides us a description of the input signal
domain $U$ (i.e. upper and lower bounds on the values appearing in the
input traces), as well as a set of time instants on which input traces
are expected to be defined (i.e. $\timedomain(\vu)$).  Initially, we
randomly sample input traces (Line~\ref{algoline:sample} and label
them as good or bad (resp.
Lines~\ref{algoline:good},\ref{algoline:bad}) depending on whether
their corresponding outputs satisfy the given $\oreq$. At the
beginning of the while-loop, we assume that there is a PSTL formula
$\psi^\candidate$ that is being considered as a candidate environment
assumption.  The first time the loop body is executed, this
enumeration occurs in Line~\ref{algoline:getNew1}, otherwise a new
\PSTL formula is obtained in the loop in Line~\ref{algoline:getNew2}.
Once we have a candidate \PSTL formula $\psi^\candidate$, we use
an off-the-shelf supervised learning approach to obtain a decision
tree $\Upsilon[\psi^\candidate]$ from $\psi^\candidate$ using a
procedure discussed in Sec.~\ref{sec:dt}. We use a procedure described
in Sec.~\ref{sec:dt2stl} to obtain an interpretable \STL formula
$\ireq^\candidate$ from $\Upsilon[\psi^\candidate]$.  If
$\ireq^\candidate$ does not give a high classification accuracy for
the given set of good/bad traces\footnote{Initially, it is possible
that we do not get any bad traces by random sampling.  In this case,
we can replace the decision tree based classifier by a procedure that
infers tight parameter valuations from only the positive examples
using approaches such as \cite{asarin2011parametric,jin2015mining}. A
potential drawback is that we may learn an environment assumption that
is narrowly applicable only to the good traces and does not generalize
well.},   we move to the next \PSTL formula to be enumerated till we
reach a user-defined upper bound on the maximum formula length. If we
exceed this bound, our procedure fails to find an accurate environment
assumption.

We note that it is possible that the candidate formula
$\ireq^\candidate$ while being accurate in classifying the set of
traces in $\traces_\good$ and $\traces_\bad$ is too permissive. This
means that it may allow for input traces not present in
$\traces_\good$ for which the corresponding output traces do not
satisfy $\oreq$. We wish to constrain the environment assumption to
exclude such signals. Thus, we invoke an off-the-shelf falsification
technique using the $\mathsf{Falsify}$ function to {\em refine} the
synthesized environment assumption. There are many promising
falsification tools such as
\cite{staliro,deshmukh_stochastic_2015,donze2010breach} that our technique
could use. The falsifier uses a global optimizer to identify an input
trace $\vu(t)$ satisfying $\ireq^\candidate$ for which $M(\vu(t))$
$\not\models \oreq$ (Line~\ref{algoline:accuracy}). Typical falsifiers
parameterize the input trace using a finite number of {\em control
points}, i.e., time points at which the signal value is deemed to be
an optimization variable. At all other time points, the intermediate
signal values are obtained through a user-specified interpolation
scheme. Let $\hat{\vu}$ denote the control point vector used by the
falsifier to generate the input trace $\vu(t)$. Then, consider an
optimizer that tries to {\em minimize} the following cost function: 
\scalebox{0.95}{
\begin{minipage}{\linewidth}

\[
\mathrm{cost}(\hat{\vu}) = (\max(0, -\rho(\ireq^\candidate,\vu,0))+1)^{2k}-1
+ \rho(\oreq,\vy,0) 
\] 

\end{minipage}}
\vspace*{7pt}

\noindent Essentially, this cost function represents a quantity that
is highly positive if the input trace does not satisfy
$\ireq^\candidate$, thus favoring input control point vectors leading
to traces that satisfy $\ireq^\candidate$. The constant $k$ is a
positive integer chosen to overpower the maximum negative robustness
that can result from the output trace $\vy(t)$ not satisfying $\oreq$.
If the input does satisfy $\ireq^\candidate$, the first term is simply
$0$, and we only look for outputs that violate $\oreq$.

If such an input trace is found, we add it to the list of bad traces
(Line~\ref{algoline:add}), and restart the enumerative solver from the
last formula that it had enumerated (Line~\ref{algoline:add}). If there
is no counterexample found, the algorithm terminates with an STL
formula representing the environment assumption. Note that our
algorithm automatically learns the structure of the environment
assumption as well as the parameter values. In the following sections,
we will explain the procedure for the decision tree based learning of
the classifier.
\begin{remark}
A key step in Algorithm~\ref{algo:cegis_ea} is systematic enumeration
of \PSTL formulas. This procedure covers the space of all \PSTL formulas. We omit the details of how this is performed, but
in essence, the procedure closely mimics the work in
\cite{mohammadinejad2019interpretable}. Longer formulas are
constructed from smaller formulas in a systematic fashion by defining
a canonical order in which \STL operators are used, and certain
efficiency improvements are added to avoid enumerating semantically
equivalent formulas with different syntax trees. More details about  systematic enumeration is provided in Appendix.
\end{remark}

\section{Supervised Learning of STL classifiers}
In this section we explain our decision-tree based algorithm for
learning the parameter valuations of a \PSTL formula that yields an
accurate \STL classifier. Before delving into the details of our procedure, we recall some related work on supervised learning of STL
formulas from data. In \cite{mohammadinejad2019interpretable}, the
authors consider a technique that enumerates {\em monotonic} \PSTL
templates and then uses the {\em validity domain boundary} of the
\PSTL formula to classify traces.

\begin{definition}[Monotonic \PSTL]
Consider a \PSTL formula $\varphi(\param)$ where $\param =
(p_1,\ldots,p_m)$.  Let $\val(\param)$ and $\val'(\param)$ be two
valuations which assign identical values to all parameters except
$p_i$. The formula $\varphi(\param)$ is called monotonically
increasing in $p_i$ if for all traces $\x(t)$, if $\x(t) \models
\varphi(\val(p_i))$ and $\val(p_i) < \val'(p_i)$, then it implies that
$\x(t) \models \varphi(\val'(p_i))$. A monotonically decreasing \PSTL
formula can be defined analogously. A \PSTL formula is called
monotonic in a parameter $p_i$ if it is either monotonically
increasing or decreasing, and is called monotonic if it is monotonic
in each of its parameters.
\end{definition}

\begin{example}
The formula $\G_{[0,10]}(\x < c)$ is monotonically increasing in $c$,
because once it is true for a given trace for some value $\val(c)$ of
$c$, it will be true for all values greater than $\val(c)$.
\end{example}

\begin{definition}[Validity Domain, Validity Domain Boundary]
The validity domain $\validitydomain \varphi$ is an open subset of
$\paramspace$ s.t.: $\forall \val(\param) \in
\validitydomain \varphi$, and for all traces $\x(t)$, $\x(t) \models
\varphi(\val(\param))$.  The boundary of the validity domain is the
set difference between the closure of the validity domain and its
interior.
\end{definition}

\begin{example}
Consider a set of traces $\x(t)$ that are all bounded above by $1$,
then for the formula $\G_{[0,10]} (\x < c)$, the validity domain is the
set $(1,\infty)$, and the validity domain boundary is the single point
$c=1$.
\end{example}

In general, computing even the validity domain boundary of a \PSTL
formula where the atomic predicates are linear inequalities of the
signals requires reasoning over semi-linear sets
\cite{asarin2011parametric}. Thus, in \cite{maler2017learning}, the
authors have proposed a multi-dimensional binary search algorithm to
approximate the validity domain boundary. In
\cite{mohammadinejad2019interpretable}, the authors propose combining
the algorithm from \cite{maler2017learning} with a supervised learning
procedure. Essentially, each step in \cite{maler2017learning}
identifies a set of points in the parameter space that lie on the
validity domain boundary. In \cite{mohammadinejad2019interpretable},
the authors propose using each successive set of points discovered by
the algorithm to define a classifier. The procedure terminates when a
sufficiently high accuracy classification is obtained. A key
limitation of this approach is that it only works for monotonic \PSTL
formulas, and when the number of parameters is high, computing the
validity domain boundary can be time-consuming.

Instead, in this paper, we consider an approach based on sampling the
parameter space, obtaining robustness values for a given set of
``seed'' traces at each of the sampled points, and using these values
as features in a decision-tree based classification algorithm. We now
explain each of these steps in detail.

\subsection{Decision Tree based Supervised Learning}

Decision trees are a non-parametric supervised learning method used
for classification and regression.  Learned trees can also be
represented as sets of if-then else rules which are understandable by
humans.  The depth of a decision tree is the length of the longest
path from the root to the leaf nodes, and the size of a decision tree
is the number of nodes in the tree.  A binary decision tree is a tree that every non-terminal node has at most two
children.  Decision trees represent a disjunction of conjunctions of
constraints represented by nodes in the tree.  Each path from the tree
root to a leaf corresponds to a conjunction of constraints while the
tree itself is a disjunction of these conjunctions
\cite{Mitchell:1997:ML:541177}.

%

While decision trees improve human readability
\cite{Mitchell:1997:ML:541177}, they are not specialized in learning
temporal properties of timed traces. A na\"{i}ve application of a
decision tree to timed traces would treat every time instant in the
trace as a decision variable, leading to deep trees that lose
interpretability.  

\begin{example}
We applied decision trees on a 2-dimensional
synthetic data set. The data set consists of two sets of traces
corresponding to signals $x$ and $y$. In both sets $y(t) = x(t-d)$,
which $d$ represents the delay between $x$ and $y$. For label $1$
traces $d<20$, and for label $0$ traces $d> 30$. Each node in decision
tree corresponds a point of $x$ and $y$ signals in time. Decision
trees failed to classify the data set properly since the resulting
tree has $179$ nodes, and the accuracy of training is $50\%$, which is
the accuracy of random classification. On the other hand, this data
set can be easily classified using \STL formula $ \varphi =
G_{[0,100]} (x(t) \geq 0.1 \implies F_{[0,20)} (y(t) \geq 0.1))$.  A
na\"{i}ve use of decision trees thus does not provide the same dynamic
richness as many temporal logic formulas.
\end{example}

Feature selection in decision trees is challenging; in our work, we
use robustness values of a given \PSTL formula at different parameter
valuations as features. For a \PSTL formula containing only one
parameter this is unnecessary, as we can simply determine the validity
domain boundary (corresponding to $0$ robustness value) by a simple
binary search. However, for \PSTL formulas with multiple independent
parameters, random samples of the parameter space can be informative
about the validity domain boundary and hence serve as features for our
decision-tree based learning algorithm. Formally,
Algorithm~\ref{alg:logical_classification} assumes that we are given
sets of traces $\traces_\good$ and $\traces_\bad$, a \PSTL formula
$\psi(\param)$ (with the parameter space $\paramspace$\footnote{$\paramspace$ is computed using upper and lower bounds on the values appearing in the input traces (e.g. in Fig.~\ref{fig:naval}, for time instances = $[0,60]$, $\paramspace = [0,60] \times [0,80]\times[15,45]$). }).
The algorithm returns the classification accuracy and the decision
tree produced by an off-the-shelf decision tree learning algorithm.
\newcommand{\train}{\mathrm{train}}
\newcommand{\test}{\mathrm{test}}
\begin{algorithm}[!t]
\small
\SetKwProg{Fn}{Function}{:}{}
\KwIn{$\psi, \traces_{\good}, \traces_{\bad}$}
\KwOut{accuracy, $\Upsilon[\psi]$}
\Fn{$\mathsf{DecisionTreeBasedSTLClassifier}(\psi, \traces_{\good}, \traces_{\bad})$}{
    \tcp{Split data for train and test}
    $\traces_\train, \traces_\test$ $\assign$ $\mathsf{split}(\traces_{\good} \cup \traces_{\bad}, 0.7)$ 
    \nllabel{algoline:split} \; 

    \tcp{Compute robustness values as features for training}
    $\featureValues_\train$ $\assign$ $\mathsf{computeFeatures}(\traces_\train,\psi,\paramspace(\psi))$ 
    \nllabel{algoline:features} \;
    \ForEach{$\vu(t) \in (\traces_\test\cup\traces_\train)$}{ 
        $\labels(\vu(t))$ $\assign$ $(\vu(t) \in \traces_{\good})$ \nllabel{algoline:label} 
    }
	    
    \tcp{Train decision tree using computed features}
    $\Upsilon[\psi^\candidate]$ $\assign$ $\mathsf{TrainDecisionTree}(\featureValues, \labels(\featureValues))$ 
    \nllabel{algoline:dt}\;
    
    \tcp{Compute accuracy}
    \ForEach{$\vu(t) \in \traces_\test$}{
        $\featureValues_\test(\vu(t))$ $\assign$ $\mathsf{computeFeatures}(\traces_\test,\psi,\paramspace(\psi))$ \;
        $\labels'(\vu(t))$ $\assign$ $\Upsilon[\psi^\candidate](\featureValues_\test(\vu(t)))$ \nllabel{algoline:newlabels}
    }
    accuracy $\assign$ $\mathsf{computeAccuracy}(\labels,\labels')$ \nllabel{algoline:accuracy} \;
    \Return accuracy, $\Upsilon[\psi]$ \;
}
\Fn{$\mathsf{computeFeatures}(\traces,\psi,\paramspace(\psi))$}{
    \tcp{Sample $m$ parameter values}
    $\paramspace_m \leftarrow \mathsf{gridSample}(\paramspace,m)$ \nllabel{algoline:gridsample} \;
    \ForEach{$\vu(t) \in \traces$}{
        \For{$i \in [1, m]$}{
            $\psi_i$ $\assign$ $\psi(\val(\paramspace_m(i)))$ \nllabel{algoline:newpsi} \;
            $\featureValues(\vu(t))[i]$ $\assign$ $\rho(\psi_i, \vu, 0)$ \nllabel{algoline:assignfeature};
        }
    }
    \Return $\featureValues$ \;
}
\caption{Classification using decision trees \label{alg:logical_classification}}
\end{algorithm}

In Line~\ref{algoline:split}, we split the given set of traces into
training and test sets; $0.7$ is an arbitrary heuristic indicating the
ratio of the size of the training set to the total number of traces.
In Line~\ref{algoline:features}, we invoke the function
$\mathsf{computeFeatures}$.  Essentially, this function maps each
trace $\vu(t)$ in the set $\traces_\train$ to a $m$-element feature
vector $\featureValues(\vu(t))$.  To produce this vector, we obtain
$m$ samples of the parameter space along a user-defined
grid\footnote{In principle, we can use $m$ random samples of the
parameter space $\paramspace$; however, in our experiments we found
that random sampling may miss parameter values crucial to obtain high
accuracy. In some sense, grid sampling {\em covers} the parameter
space more evenly leading to better classification accuracy. In our experiments, $4 \leq m \leq 10$ samples is sufficient to get a high accuracy.}. Note
that the grid sampling procedure also checks for validity of a
parameter sample; e.g. if $\tau_1$ and $\tau_2$ are parameters
belonging to the same time-interval $[\tau_1,\tau_2]$, then it imposes
that $\tau_1 < \tau_2$. See Fig.~\ref{fig:grid_sampling} for an
example of grid sampling. Each sample in the parameter space
corresponds to a valuation for the parameters in the \PSTL formula
$\psi$, and applying the $i^{th}$ valuation yields the \STL formula
$\psi_i$ (Line~\ref{algoline:newpsi}). We then use the robustness
value of $\vu(t)$ w.r.t. $\psi_i$ as the $i^{th}$ element of the
feature vector, i.e.  $\featureValues(\vu(t))[i]$. For each trace in
the set $\traces_\train$ and  $\traces_\test$ , we assign it label $1$ if it belongs to
$\traces_\good$, and $0$ otherwise (Line~\ref{algoline:label}). 

In Line~\ref{algoline:dt}, we invoke the decision tree procedure on the
feature vectors and the label sets. The edge between any node in the
decision tree $\Upsilon[\psi]$ and its children is annotated by a
constraint of the form $\rho(\psi_i,\vu,0) < c$ for the left child,
and its negation for the right child. Here, $c$ is some real number.
We give further details on the structure of the tree in the
Section~\ref{sec:dt2stl}. Next, we compute the accuracy of the
decision tree by computing the labels of the traces in the test
set $\traces_\test$ and comparing them to their ground truth labels.
The function $\mathsf{computeAccuracy}$ simply computes the ratio
$|\{\vu(t) \mid \labels(\vu(t)) =
\labels'(\vu(t))\}|\,\slash\,|\traces_\test|$.

\begin{figure}[!t]
\centering
\includegraphics[scale=0.2]{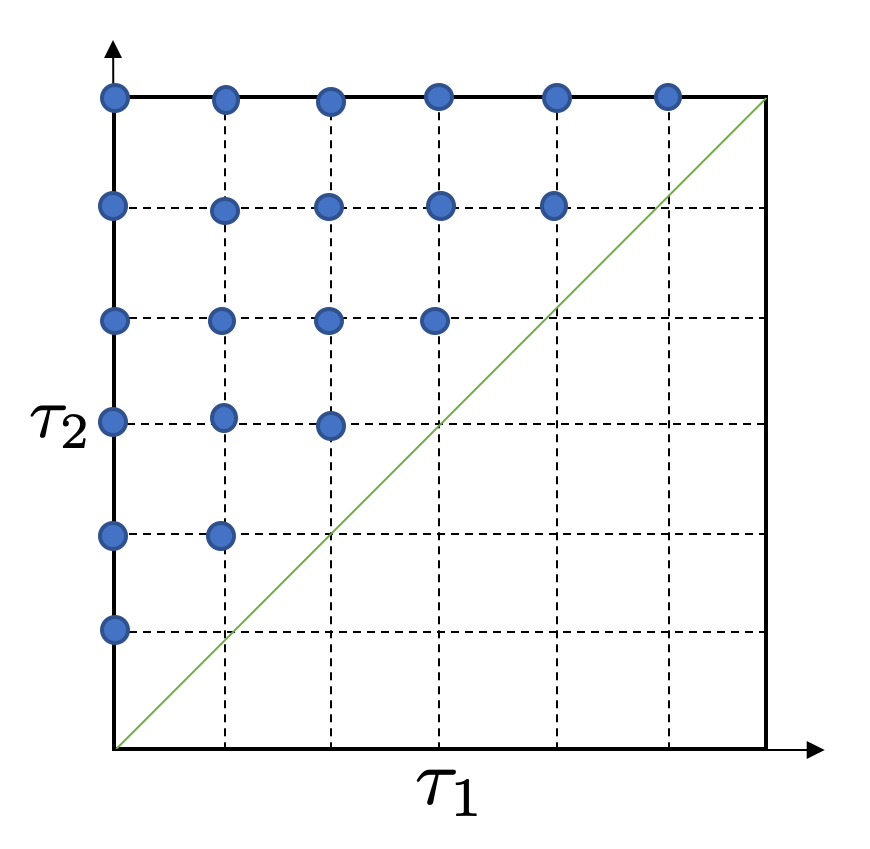}
\caption{Grid sampling of time parameters for formula $\G_{[\tau_1,\tau_2]}(x(t) > c)$. Since $\tau_1$ should be less than $\tau_2$, the area above $\tau_1 = \tau_2$ line (green line) is sampled.}
\label{fig:grid_sampling}
\end{figure}

\section{Extracting Interpretable STL formulas}
\begin{figure}[t]
\centering
\begin{tikzpicture}
\tikzstyle{smalltext}=[font=\fontsize{8}{8}\selectfont]

\node[circle,draw] (root) {$1$};
\node[circle,draw,below left of=root,node distance=20mm] (00) {$2$};
\node[circle,draw,below right of=root,node distance=20mm] (01) {$3$};
\node[circle,draw,below left of=00,node distance=20mm] (000) {$4$};
\node[circle,draw,below right of=00,node distance=20mm] (001) {$5$};
\node[rectangle,draw,fill=green!10,below of=01,node distance=15mm] (l1) {$\ell = 1$}; 
\node[rectangle,draw,fill=green!10,below of=000,node distance=15mm] (l0) {$\ell = 0$}; 
\node[rectangle,draw,fill=green!10,below of=001,node distance=15mm] (l11) {$\ell = 1$}; 

\draw[->] (root) -- node[left,smalltext] {$\rho(\psi_1,\vu,0) < c_1$} (00);
\draw[->] (root) -- node[right,smalltext] {$\rho(\psi_1,\vu,0) \ge c_1$} (01);
\draw[->] (00) -- node[left,smalltext] {$\rho(\psi_2,\vu,0) < c_2$} (000);
\draw[->] (00) -- node[right,pos=0.2,smalltext] {$\rho(\psi_2,\vu,0) \ge c_2$} (001);
\draw[->] (01) -- (l1);
\draw[->] (000) -- (l0);
\draw[->] (001) -- (l11);

\end{tikzpicture}

\caption{Example Tree returned by 
$\mathsf{DecisionTreeBasedSTLClassifier}$\label{fig:dt}}
\end{figure}
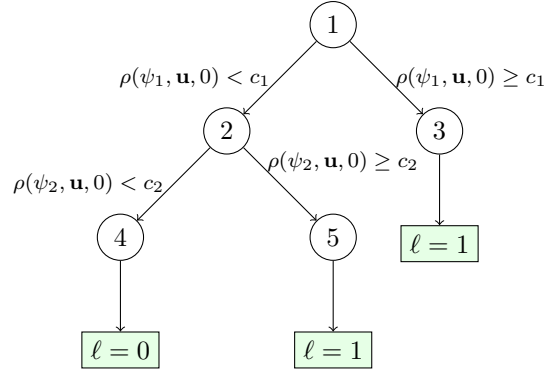

The function $\mathsf{DecisionTreeBasedSTLClassifier}$ described in
Algo.~\ref{alg:logical_classification} returns a decision tree
$\Upsilon[\psi]$ of the form shown in Fig.~\ref{fig:dt}. We note that
the edge labels correspond to inequality tests over \STL formulas
$\psi_i$ corresponding to the same \PSTL formula $\psi$, but with
different valuations for the parameters. Each path from the root of
the tree to a leaf node represents a conjunction of the edge labels,
and the disjunction over all paths leading to the same label
represents the symbolic condition for mapping a given trace to a
given label. Paths leading to label $\ell=1$ correspond to the
environment assumption that we wish to mine. We now show that given a
decision tree of this form, it is always possible to extract an \STL
formula from the symbolic condition that the decision tree represents.

\begin{lemma}
\label{lem:const}
For any \STL formula $\varphi$, any trace $\vu(t)$, and any time
instance $t$, for $\sim \in \{<,\ge,>,\le\}$, any constraint of the
form $\rho(\varphi,\vu,t) \sim c$ can be transformed to the
satisfaction or violation of a formula $\hat{\varphi}$ by $\vu(t)$,
where $\hat{\varphi}$ can be obtained from $\varphi$ and $c$ using
simple transformations (shifts in space parameters).
\end{lemma}

\begin{proof}
We prove the above lemma using structural induction on the syntax of
\STL.  The base case is for atomic predicates.  Suppose $\varphi =
f(\vu) > \sPara$, then if $\rho(\varphi, \vu, t) > c$, by the
definition of a robustness value, $f(\vu(t)) - (\sPara + c) > 0$. Let
$\hat{\varphi} = f(\vu) > (\sPara + c)$.  Then, $\rho(\varphi,\vu,t) >
c$ implies that at time $t$, $\vu(t) \models \hat{\varphi}$. The proof
for atomic predicates indicating other kinds of inequalities is
similar.

The inductive hypothesis is that the above lemma holds for all proper
subformulas of $\varphi$, and in the inductive step we show that if
this is true, then the lemma holds for $\varphi$. 
\noindent (1) Let $\varphi = \neg \psi$. Then, $\rho(\neg\psi,\vu,t) >
c$ implies that $-\rho(\psi,\vu,t) > c$, or $\rho(\psi,\vu,t) \le -c$. 
Let $c' = -c$, then by the inductive hypothesis, there is a formula
$\hat{\psi}$ such that $\rho(\hat{\psi},\vu,t) \le 0$.  \\

\noindent (2) Let $\varphi = \psi_1 \wedge \psi_2$. If $\rho(\psi_1
\wedge \psi_2, \vu, t) > c$, then
$\min(\rho(\psi_1,\vu,t),\rho(\psi_2,\vu,t)) > c$, which implies that
$\rho(\psi_1,\vu,t) > c$ and $\rho(\psi_2,\vu,t) > c$. Again, by the
inductive hypothesis, this implies that there are formulas
$\hat{\psi_1}$  and $\hat{\psi_2}$ such that $\rho(\hat{\psi_1},\vu,t)
> 0$ and $\rho(\hat{\psi_2},\vu,t) > 0$. This implies that
$\min(\rho(\hat{\psi_1},\vu,t),\rho(\hat{\psi_2},\vu,t))>0$, or
$\rho(\hat{\psi_1} \wedge \hat{\psi_2}, \vu, t) > 0$. \\

\noindent (3) Let $\varphi = \psi_1 \vee \psi_2$. An argument similar
to (2) can be used to prove that we can obtain $\hat{\psi_1}$ and
$\hat{\psi_2}$ such that $\rho(\hat{\psi_1}\vee\hat{\psi_2},\vu,t) >
0$. \\

\noindent (4) Let $\varphi$ = $\G_\Intvl \psi$. $\rho(\G_\Intvl \psi,
\vu, t) > c$ implies that $\forall t' \in t \oplus \Intvl$,
$\rho(\psi, \vu, t) > c$. Following similar reasoning as (2), and
using the inductive hypothesis, we can show that there exists an \STL
formula $\hat{\psi}$ such that the above is equivalent to
$\rho(\G_\Intvl \hat{\psi}, \vu, t) > 0$. \\

\noindent (5) For $\varphi$ =  $\F_\Intvl \psi$, and $\varphi = \psi_1
\U_\Intvl \psi_2$ similar reasoning as (4) can be used. We omit the
details for brevity.

Finally, we can have a similar proof for any constraint of the form
$\rho(\varphi,\vu,t) < c$. For example, consider $\varphi = \psi_1
\wedge \psi_2$. $\rho(\varphi,\vu,t) < c$ implies that
$\min(\rho(\psi_1,\vu,t), \rho(\psi_2,\vu,t))< c$, which in turn
implies that $\rho(\psi_1,\vu,t) < c$ or $\rho(\psi_2,\vu,t) < c$. By
the inductive hypothesis we can obtain $\hat{\psi_1}$ and
$\hat{\psi_2}$ such that $\rho(\hat{\psi_1},\vu,t) < 0$ or
$\rho(\hat{\psi_2},\vu,t) < 0$, which implies that
$\rho(\hat{\psi_1}\wedge \hat{\psi_2},\vu,t) < 0$.

As we are able to prove the inductive step for any kind of \STL
operator, and for all types of constraints on the robustness value, by
combining the different cases, we can conclude that the lemma holds
for an arbitrary \STL formula.
\end{proof}

\begin{theorem}
\label{theo1}
Given a decision tree $\Upsilon[\psi]$ where edge labels denote
constraints of the form $\rho(\psi_i,\vu,0) > c_i$, we can obtain an
\STL formula that is satisfied by all input traces that are labeled
$1$ by the decision tree. 
\end{theorem}

\begin{proof}
The proof follows from the proof of Lemma~\ref{lem:const}.
Essentially, each constraint corresponding to an edge label can be
transformed into an equivalent \STL formula, and each path is a
conjunction of edge labels; so each path gives us an \STL formula
representing the conjunction of formulas corresponding to each edge
label. Finally, a disjunction over all paths corresponds to a
disjunction over formulas corresponding to each path.
\end{proof}

\begin{remark}

We note that the above procedure does not require the \PSTL formula to
be monotonic.  If the chosen \PSTL formula is monotonic,  then it is
possible to simplify the formula further. Essentially, along any path,
we can retain only those formulas corresponding to parameter
valuations that are incomparable according to the order imposed by
monotony. Furthermore, each of these valuations corresponds to points
on the validity domain boundary as the robustness value for these
valuations is close to zero. We also remark that Lemma~\ref{lem:const}
gives us a constructive approach to build an \STL formula from the
decision tree -- we simply need to follow the recursive rules to push
the constants appearing in the inequalities on the robustness values
to the atomic predicates.

\end{remark}

\section{Benchmarking Supervised Learning}
We divide our evaluation of the techniques presented in this paper
into two parts. In this section, we primarily benchmark the efficacy
of our decision tree based supervised learning approach. In the next
section, we discuss case studies of mining environment assumptions
using a combination of enumerative structure learning of the \PSTL
formula with the decision tree based classification approach. We run the experiments on an Intel Core-i7 Macbook Pro
with 2.7 GHz processors and 16 GB RAM and used decision tree algorithms from Statistics and Machine Learning Toolbox in Matlab with default parameters.

\mypara{Maritime Surveillance} We compare the results of
classification with our tool with the DTL4STL tool
\cite{bombara2016decision}. For a fair comparison, we use the same
data set used by \cite{bombara2016decision}. The maritime surveillance
data set is a 2-dimensional synthetic data set consists of three types
of behaviors: one normal and two anomalous behaviors (see
Fig.~\ref{fig:naval}). 

We applied our tool to 600 traces from this data set (300 traces for
training and 300 traces for testing). The \STL formulas learned by our
technique are as follows:
\begin{align*}
&\varphi_{green} = \neg\varphi_1 \aand \varphi_2\\
&\varphi_{blue} = \neg\varphi_2\\
&\varphi_{red} = \varphi_1 \aand \varphi_2,\\
\end{align*}

where $\varphi_1 = \G_{[15,30]} (x(t) < 39)$ and $\varphi_2 =
\G_{[30,45]} (x(t) < 41.98)$. $\varphi_{green}$ is the formula for
classification of green traces from the others (red and blue traces).
$\varphi_{blue}$ and $\varphi_{red}$ classify blue and red traces from
the others respectively. The train accuracy is $100\%$ and, the test
accuracy is $99\%$ with  training time = 24.82 seconds.The simplest
\STL formula learned by DTL4STL \cite{bombara2016decision} to classify
green traces from the others is:
\begin{align*}
&\varphi = (\varphi_1 \aand (\neg\varphi_2 \vee (\varphi_2 \aand \neg \varphi_3))) \vee (\neg \varphi_1 \aand (\varphi_4 \aand \varphi_5))\\
&\varphi_1=\G_{[199.70,297.27)}(\F_{[0.00,0.05)} (x[t] \leq 23.60)\\
&\varphi_2=\G_{[4.47,16.64)}(\F_{[0.00,198.73)} (y[t] \leq 24.20)\\
&\varphi_3=\G_{[34.40,52.89)}(\F_{[0.00,61.74)} (y[t] \leq 19.62)\\
&\varphi_4=\G_{[30.96,37.88)}(\F_{[0.00,250.37)} (x[t] \leq 36.60)\\
&\varphi_5=\G_{[62.76,253.23)}(\F_{[0.00,41.07)} (y[t] \leq 29.90)\\
\end{align*}
with the average misclassification rate of $0.007$. This \STL formula
is long and complicated compared to the \STL formula $\varphi_{green}$
learned by our framework. Long formulas hinder interpretability and
are not desirable for describing time-series behaviors. The reason
behind generating complicated formulas by DTL4STL
\cite{bombara2016decision} is the restriction to only eventually and
globally as \PSTL templates. Our technique considers the space of all
\PSTL formulas in increasing order of complexity which results in
simple and interpretable \STL classifiers.

\begin{figure}[!t]
\centering
\includegraphics[scale=0.2]{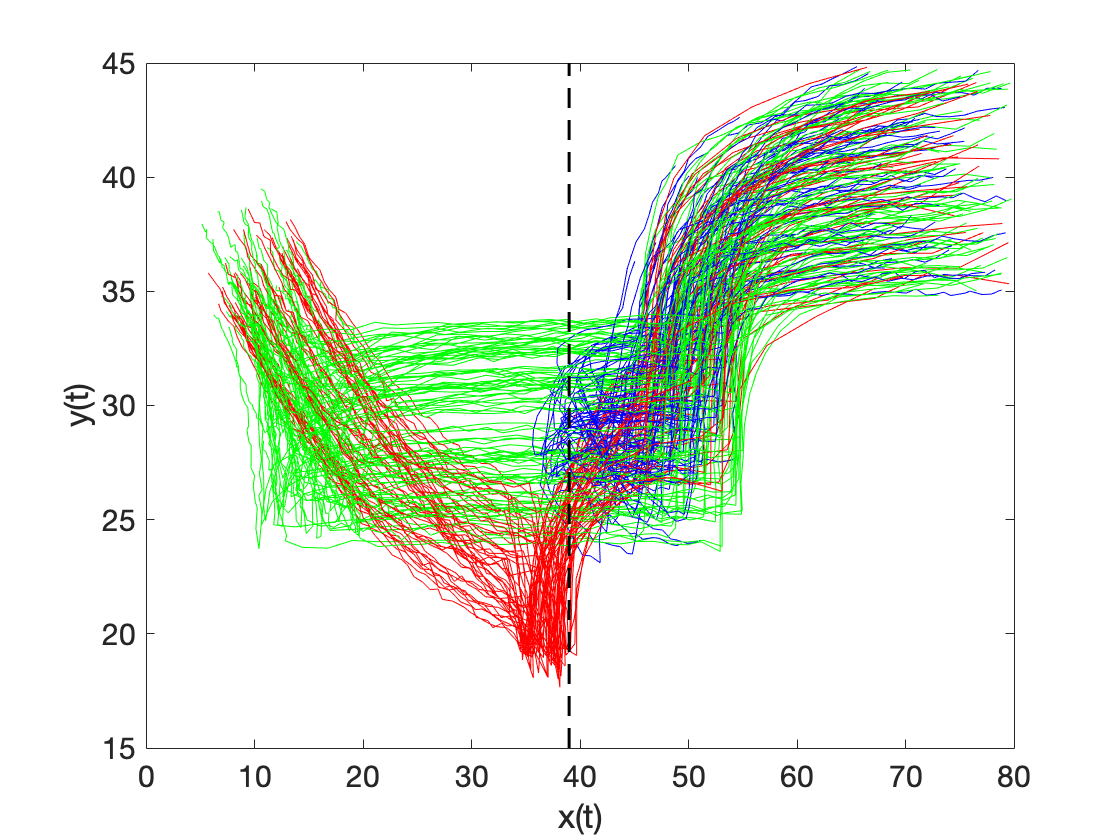}

\caption{Naval surveillance data set \cite{bombara2016decision}(Green
traces: normal trajectories, red and blue traces: two kinds of
anomalous trajectories and the dashed line indicates one of the space parameters learned by our procedure.}
\label{fig:naval}
\end{figure}

\mypara{ECG5days} We also applied our technique on ECG five days data
set from the UCR time-series repository \cite{chen2015ucr}. The data
set consists of echo-cardiogram signals recorded from a 67 year old
male. The two classes correspond to two dates that the ECG was
recorded, which are five days apart. We used 300 traces from this data
set for training and 300 for testing.  The \STL formula learned by our
method to classify two classes of ECG behaviors is:
\begin{align*}
&\varphi = \neg\varphi_1 \wedge (\varphi_2 \vee \neg \varphi_3) \\
&\varphi_1=\G_{[19.28,57.86]}(\F_{[19.28,38.57]} (x[t] > 0.48) \\
&\varphi_2=\G_{[38.57,57.86]}(\F_{[19.28,57.86]} (x[t] > 1.87) \\
&\varphi_3=\G_{[38.57,57.86]}(\F_{[19.28,57.86]} (x[t] > 1.32) \\
\end{align*}
with training accuracy = $99\%$ and testing accuracy = $96\%$. The
required time for training is 1449.34 seconds. 

\begin{figure}[!t]
\centering
\includegraphics[scale=0.2]{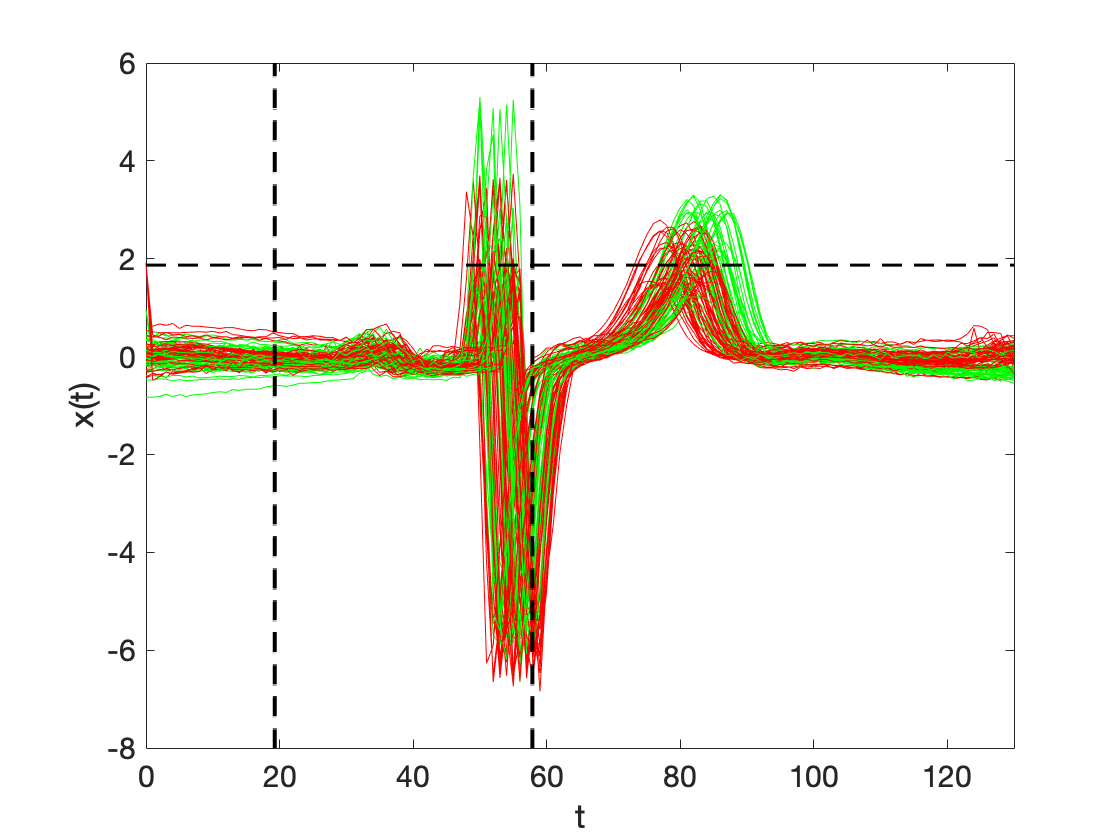}

\caption{ECG five days data set from UCR time-series repository
\cite{chen2015ucr}. The two classes correspond to two dates that the
ECG signals were recorded. Dash lines illustrate the time and space parameters
of the learned \STL formula.}
\label{fig:ecg5days}
\end{figure}

\section{Case Studies}
In this section, we benchmark mining environment assumptions on a few case studies. The first is a synthetic model that we hand-crafted to
demonstrate learning \STL assumptions.  The second is a model of an
automatic transmission system. A model of abstract fuel control is considered as the third case study.

\subsection{Synthetic model}

\simulink is a visual block diagram language commonly used in
industrial settings to model component-based CPS designs. We created a
\simulink model of an oscillator component that has two input signals $u_1$ and $u_2$ and an output
signal $y$. The model has an internal flag that is turned on if within
$3$ seconds of the input value $u_1$ falling below $0$, the input
value $u_2$ also falls below $0$.  When the flag is turned on, the
oscillator outputs a sinusoidal wave with amplitude $5$ units, and
outputs a sine wave with amplitude $1$ unit otherwise.
We imagine a scenario where a downstream component requires the output
of the oscillator component to be bounded by $[-1,1]$. I.e., we
require that the output $y$ satisfies the STL requirement $\oreq =
\alw (-1 \leq y(t) \leq 1)$.  

In this example, we generate a large number of input traces using a
$\simulink$ based signal generator. We pick a small subset of these
traces that includes input traces that both lead to outputs satisfying
$\oreq$ (i.e. the good traces), and violating $\oreq$ (bad traces).
We use our supervised learning framework to learn an STL classifier
$\ireq$.  We then invoke the counterexample-guided refinement step of
Algorithm~\ref{algo:cegis_ea} to improve $\ireq$.  We learned the
environment assumption $\ireq =\alw_{[0,20]} (u_1(t) < 0 \implies
\G_{[0,5]} (u_2(t) \geq 0))$. This means that when input $u_1$ becomes
negative, $u_2$ should stay non-negative within $[0,5]$ seconds.
Otherwise, the output will violate $\oreq$.  The time taken to learn
this $\ireq$ is 6084 seconds  and the training and testing accuracies are $100\%$ respectively. We note that in this case, the learned
formula $\ireq$ is stronger than the theoretical environment
assumption that we had in mind when designing the model. This
discrepancy can be due to the reason that our training set did
not include trajectories where $u_1(t_1) < 0$ and $u_2(t_2) < 0$
occurred when $3 < t_2 - t_1 < 5$.



\mypara{Automatic Transmission Controller} We consider automatic
transmission controller which is a built-in model in \simulink, shown
in Fig.~\ref{fig:auto_trans_simulink} in Appendix. This model consists of modules
to represent the engine, transmission, the vehicle, and a shift logic
block to control the transmission ratio. User inputs to the model are
throttle and brake torque. Engine speed, gear and vehicle speed are
outputs of the system. We are interested in the following signals: the
throttle, the vehicle speed, and the engine speed measured in RPM
(rotations per minute). We wish to mine the environment assumptions on
the throttle that ensures that the engine speed never exceeds 4500 rpm, and
that the vehicle never drives faster than 120 mph. In other words, we
want to mine the \STL specification $\ireq$ on input of the system
(throttle) that results in meeting the following output
requirement:
\[\oreq = \G ( RPM \leq 4500) \aand \G (speed \leq 120). \]

A set of traces that violate this requirement is shown in
Fig.~\ref{fig:auto_trans}. We applied our assumption mining method on
600 throttle traces (300 for training and 300 for testing). The
formula produced by our framework is $\ireq = \G_{[240,480]}(x(t) <
40.4281)$ with training and testing accuracy equal to $100\%$ and
$98\%$ respectively. This formula implies that if the throttle stays
below $40.4281$ in time interval $[240,480]$, the engine and vehicle
speed will meet the requirement. Otherwise, engine or vehicle speed
violate the specifications and go beyond the specified threshold. It
is difficult to mine such behaviors by looking at input and output
traces of the system, and our technique helped in mining such
assumptions on input automatically. The training time for learning the
\STL formula is 28.18 seconds.


\begin{figure}[!t]
\centering
\includegraphics[scale=0.2]{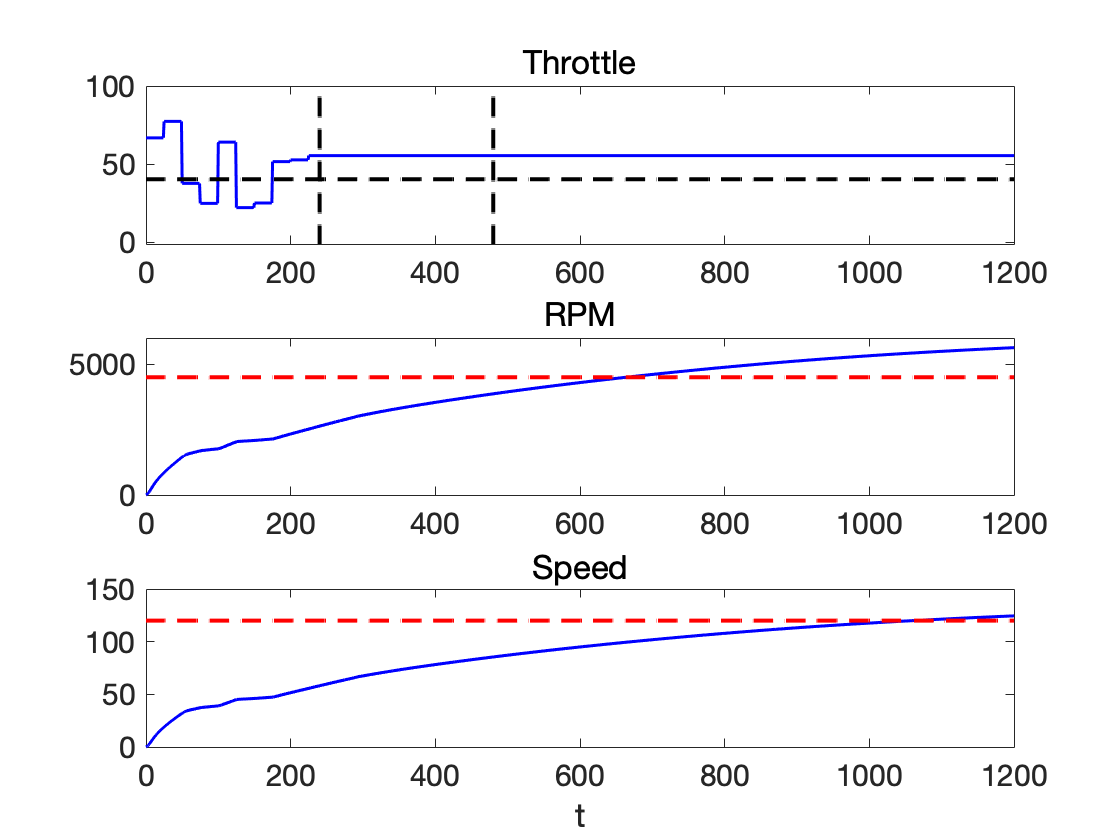}
\caption{Violating traces for the automatic transmission controller.}
\label{fig:auto_trans}
\end{figure}

\mypara{Abstract Fuel Control Model} 
In \cite{jin2014powertrain}, the authors provide a $\simulink$ model for a power train control system. The model takes as input the throttle angle and engine speed, and outputs the Air-to-Fuel ratio. As specified in \cite{jin2014powertrain}, the authors indicate that it is important for the A/F ratio to stay within $10\%$ of the nominal value. The output requirements in \cite{jin2014powertrain} are applicable only in the normal mode of operation of the model, when the throttle angle exceeds a certain threshold, the model switches into the power mode. In this mode, the A/F ratio is allowed to be lower. We wanted to extract the assumptions on the throttle angle that lead to significant excursions from the stoichiometric A/F value. We were able to learn the formula $\alw_{[0,100]} (x(t) < 61.167)$  which confirms that the excursions happen when the model goes into the power mode. We were able to synthesize the environment assumptions in 13.40 seconds with both training and test accuracy accuracy of $100\%$.

\section{Related Work}
\mypara{Learning from timed traces} In the machine learning (ML)
community, various supervised learning methods for timed traces have
been proposed. Traditionally, ML techniques rely on large sets of
generic features (such as those based on statistics on the values
appearing in the timed traces of features obtained through signal
processing).  A key drawback of ML approaches is the lack of
interpretability of the classifiers.

Partially to address interpretability, there have been significant
recent efforts at learning temporal logic formulas from data.  There
is work on learning STL formulas in a supervised learning context
\cite{bombara2016decision,bombara2018online,rPSTL,asarin2011parametric,bartocci2013learning,nenzi2018robust,gol2018efficient,ketenci2019synthesis},
passive learning \cite{Jha2017}, an unsupervised learning context
\cite{jones2014anomaly,vazquez2018time,vazquez2017logical}, and in an
active learning setting \cite{juniwal2014cpsgrader}. Especially
relevant to this paper is the seminal work in
\cite{bombara2016decision}, where the authors propose learning the
structure and parameters of \STL formulas using decision trees.  In
contrast to our technique where there is a single \STL formula used
throughout the decision tree, in \cite{bombara2016decision}, each node
is associated with a primitive \PSTL formula. The technique then makes
use of impurity measures to rank the primitives according to how
accurately they label the set of traces (compared to ground truth).
The primitives come from a fragment of \PSTL containing formulas with
only top-level $\F$, $\G$, $\F\G$ or $\G\F$ operators.  We observe
that the generated \STL formulas in this approach can become long and
complicated, especially because each node in the decision tree can
potentially be a different \STL formula. The decision trees produced
by this method lead to formulas that splice together local deductions
over traces together into a bigger formula.  

\mypara{Requirement Mining} In
\cite{jin2015mining,chen2016active,hoxha2018mining}, the authors
address the problem of mining (output) requirements. Here, they assume
that the structure of the \PSTL formula representing an output
requirement is provided by the user.  The technique then uses
counterexample guided inductive synthesis to infer formula parameters
that result in an \STL formula that is satisfied by all observed model
outputs. Key differences from this method are: (1) we are interested
in mining environment assumptions and not output requirements, (2) we
use a supervised learning procedure that separates input traces that
lead to outputs satisfying/violating an output requirement. The work
in \cite{li-dac10} focuses on mining temporal patterns from traces of
digital circuits, and uses automata-based ideas that are quite
different from the work presented here.

The seminal work proposed by Ferr{\`e}re et al. \cite{ferrere2019interface} extends \STL with support to define
input-output interfaces. A new measure of input vacuity and output
robustness is defined that better reflect the nature of the system and
the specification intent.  Then, the robustness computation is adapted
to exploit the input-output interface and consequently provide deeper
insights into the system behaviors. The connection of this work with
our technique is that we also look at input-output relations using
\STL specifications. Our method mines a \STL formula on input that guarantees a desired requirement on outputs. It would be interesting
to extend the methods developed in this paper to the problem of mining
interface-aware \STL requirements. The latter is more expressive as it
includes predicates that combine input and output variables, which our
current paper does not address.

In \cite{DiwakaranSriram2017Analyzing}, the authors analyze the
falsifying traces of a cyber-physical system. Concretely, they seek to
understand the properties or parts of the inputs to a system model
that results in a counterexample using sensitivity analysis. They use
learning methods (such as statistical hypothesis testing) from
repeated simulations for the system under test. Tornado diagrams are
used to find the values till no violation occurs while SMT solvers are
used to find the falsifying intervals. Our work in this paper can be
used to solve a similar problem, by basically mining environment
assumptions for $\neg \oreq$ for a given output requirement. A key
difference in our technique is that we seek to explain falsifying
input traces using an \STL formula, while the work in
\cite{DiwakaranSriram2017Analyzing} formulates explanations directly
in terms of the input traces.

\section{Conclusions}
In this work, we addressed the problem of
mining environment assumptions for CPS components and representing them
using Signal Temporal Logic. An input trace satisfying an environment
assumption is guaranteed to produce an output that meets the component
requirement. We use a counterexample-guided procedure that
systematically enumerates parametric STL formulas and uses a decision
tree based classification procedure to learn both the structure and
precise numeric constants of an STL formula representing the
environment assumption. We demonstrate our technique on a few benchmark CPS models.

\bibliographystyle{unsrt}  
\bibliography{references}  

\clearpage
\newpage
\section*{Appendix}
\label{sec:appendix}
\begin{figure}[!t]
\centering
\includegraphics[scale=0.1]{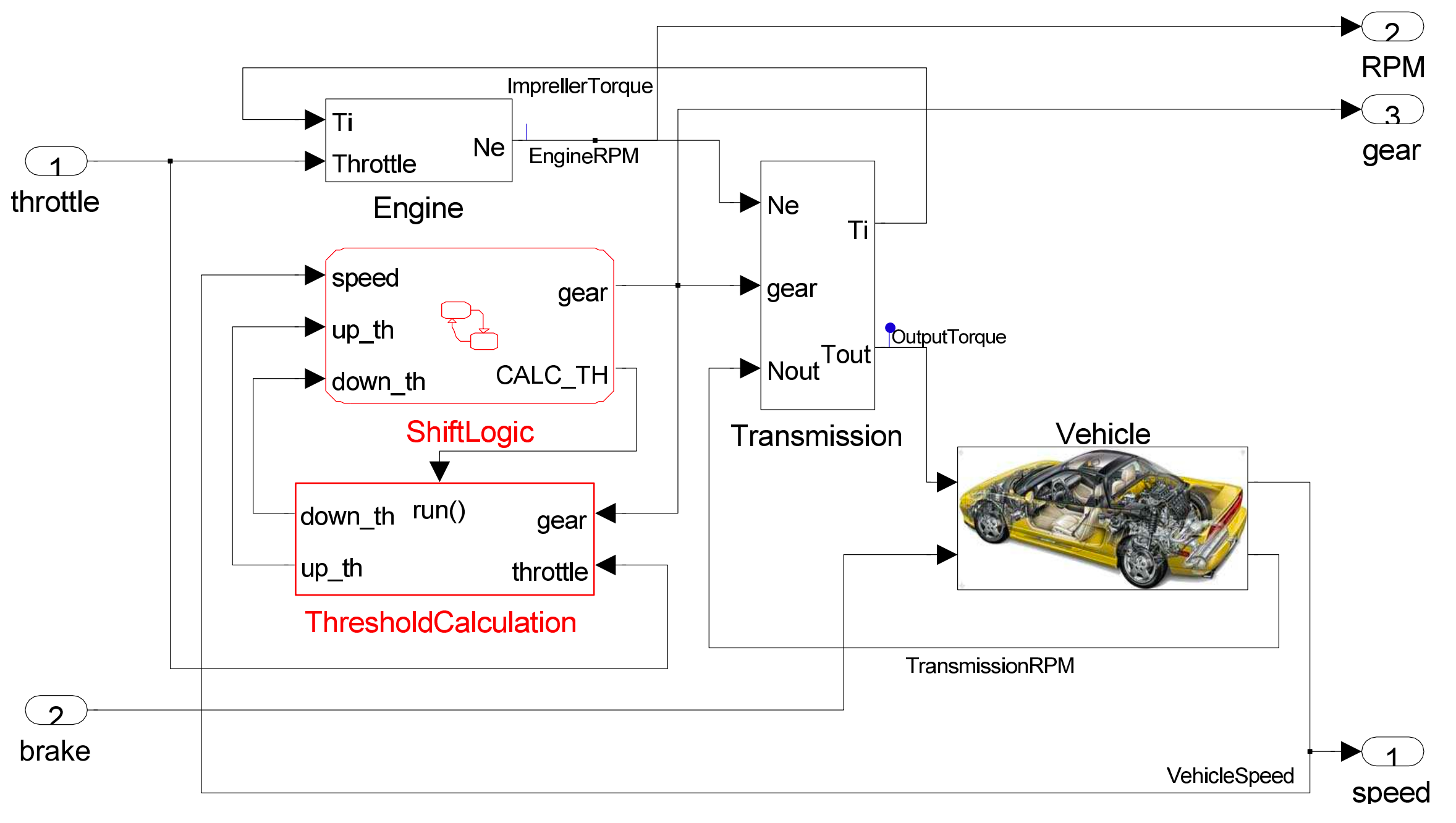}
\caption{The Simulink model of automatic transmission controller. Inputs of the system are throttle and brake. RPM, gear and speed are outputs of the system.}
\label{fig:auto_trans_simulink}
\end{figure}

\mypara{Systematic enumeration \cite{mohammadinejad2019interpretable}} 
From a grammar-based perspective a PSTL formula can be viewed as
atomic formulas combined with unary or binary operators.  For
instance, PSTL formula $ (x(t) > c_1) \Rightarrow \alw_{[0,\tau_2]} (x(t)< c_2)$ consists of binary operator  $ \Rightarrow$, unary operator $\alw$, and atomic predicates $x(t) > c_1$ and
$x(t)< c_2$. Systematic enumeration algorithm consists of the following tasks:
\begin{enumerate}
\item First, basically, all formulas of length $1$, or parameterized signal predicates are enumerated.
\item All enumerated formulas are stored in a database sorted in non-decreasing order of their length.

\item Unary and Binary operators in a user-defined order are applied on all previously enumerated formulas.

\item As the space of all STL formulas is very large,
and contains many semantically equivalent formulas, an optimization technique is proposed to prune the space of formulas considered.
\end{enumerate}

For each PSTL formula $\psi^\candidate$ generated by Systematic enumeration algorithm \cite{mohammadinejad2019interpretable}, we apply the procedure which is formalized in Algo.~\ref{algo:cegis_ea}. If  $\psi^\candidate$ is a good environment assumption (high accuracy), algorithm terminates and $\psi^\candidate$ is
returned. Otherwise, the procedure continues to generate new PSTL
formulas.

\end{document}